\documentclass[11pt]{article}

\usepackage[margin=1in]{geometry}
\usepackage{amsmath,amssymb,amsthm}
\usepackage{graphicx}
\usepackage{booktabs}
\usepackage{algorithm}
\usepackage{algpseudocode}
\usepackage[colorlinks=true,linkcolor=blue,citecolor=blue,urlcolor=blue]{hyperref}
\usepackage{natbib}
\usepackage{xcolor}
\usepackage{pgfplots}
\pgfplotsset{compat=1.17}
\usetikzlibrary{positioning}

\newtheorem{proposition}{Proposition}
\newtheorem{definition}{Definition}
\newtheorem{remark}{Remark}
\newtheorem{corollary}{Corollary}

\title{\textbf{Dynamic Coalition Formation and Communication Pricing\\
in Skill-Based Agentic AI Systems}}

\author{
  Mojtaba Eslami$^{1}$ \\[4pt]
  \small $^{1}$University of Calgary,  
  \small \texttt{mojtaba.eslami@alumni.ucalgary.ca}
}

\date{}

\begin{document}

\maketitle

\begin{abstract}
Modern agentic AI systems compose multiple large language model (LLM) agents with
heterogeneous skills to solve complex tasks, but most deployed architectures either
fix the communication topology in advance or let every agent broadcast to every other
agent. Both choices are economically inefficient: token cost, latency, and error
propagation all grow with the number of active agents and communication edges, while
the marginal benefit of an additional agent or message is often small or negative. We
formalize agent coalition formation and inter-agent communication as a cooperative
game with a task-conditioned net-utility function $U(C\mid x) = V(C\mid x) -
\sum_{i\in C} c_i$ that cleanly separates coalition-level costs (latency,
hallucination risk, redundancy) from each agent's own activation cost, so that token
and compute cost is accounted for exactly once. Building on
this, we propose (i) a marginal-value activation rule and greedy router that decides
which agents to invoke, (ii) an extension that treats communication edges themselves
as decision variables subject to per-edge cost, and (iii) an online routing mechanism
that uses \emph{estimated} Shapley values -- rather than only post-hoc credit
assignment -- to predict which agents are worth contacting before and during
execution. We connect the resulting optimization problem to submodular set-function
maximization and prove two correctly-scoped guarantees: a curvature-refined bound for
a cardinality-constrained, monotone special case, and a tight $1/2$-approximation
(with an explicit correction for signed objectives) for the unconstrained,
non-monotone case via a double-greedy procedure -- and we are explicit that neither
guarantee applies directly to our main router, which is presented as a heuristic
motivated by, but not proven equal to, these bounds. We further prove a
\emph{Shapley--submodularity sandwich bound} showing that the gap between any
marginal-value routing estimate and an agent's fair Shapley credit is bounded by a
single, per-agent diminishing-returns quantity. In a controlled synthetic simulation,
greedy marginal-value routing recovers $99.5\%$ of brute-force-optimal utility while
activating on average $1.96$ of $8$ candidate agents, versus $38.8\%$ of optimal
utility for full broadcast; a sensitivity analysis shows this result is robust to
activation cost and redundancy weight, but degrades substantially (to as low as
$66\%$) when the submodularity assumption is violated or value estimates are noisy --
exactly the two failure modes our theory predicts should matter, and, we argue, a more
informative characterization of the method than the single headline number. We situate
our contribution narrowly relative to closely neighboring work on Shapley-based
pricing, hedonic-game coalition stability, and communication-graph pruning for LLM
agents, and outline a concrete empirical protocol -- using real multi-agent LLM
benchmarks -- for validating the framework beyond simulation.
\end{abstract}


\section{Introduction}
\label{sec:intro}

Agentic AI systems increasingly decompose a task among multiple specialized LLM
agents -- planners, coders, retrievers, critics, and verifiers -- that communicate to
produce a final answer \citep{guo2024survey}. As these systems scale, two failure
modes recur. First, \emph{over-activation}: frameworks that invoke every available
skill or tool regardless of whether it improves the outcome, inflating token cost and
latency without a commensurate quality gain, and sometimes actively hurting quality
through conflicting or redundant outputs \citep{cemri2025multiagent}. Second,
\emph{over-communication}: architectures with a fixed, dense communication topology
(e.g., every agent sees every other agent's output) that do not adapt to the marginal
informativeness of a message for a given task
\citep{zhang2025agentprune,zhuge2024gptswarm}.

Our central claim is that both problems admit a common formalization:
\begin{quote}
\emph{An agent should be activated, and a communication edge should be used, only
when the expected marginal value it creates for the task exceeds its computational,
latency, privacy, and error-propagation cost.}
\end{quote}
This reframes agent orchestration as an economic allocation problem, and it lets us
import tools from cooperative game theory -- characteristic functions, the core,
hedonic games, and the Shapley value -- together with mechanism design, to reason
about \emph{which} agents should be invoked, \emph{how} they should be connected, and
\emph{how} credit (and hence future routing probability or payment) should be
assigned once a task is complete. Several of these ideas individually have close
neighbors in the literature (Section~\ref{sec:related}): Shapley-based pricing for
LLM agents, hedonic-game coalition stability, and communication-graph pruning have
each been studied separately. Our defensible claim of novelty is narrower than any
one of these: \emph{a task-conditioned net-utility formulation that jointly links
coalition selection, communication cost, and predictive contribution-based routing,
together with a diagnostic bound connecting context-specific marginal contributions
to Shapley credit under submodularity.}

\paragraph{Core contributions.} We make three primary contributions.
\begin{enumerate}
  \item \textbf{A task-conditioned net-utility formulation}, $U(C\mid x) = V(C\mid x)
  - \sum_{i\in C} c_i$, that separates coalition value (quality net of latency, risk,
  and redundancy costs) from per-agent activation cost, together with a joint
  extension in which the communication graph over an active coalition is optimized
  alongside coalition membership (Sections~\ref{sec:formulation}--\ref{sec:comm}).
  \item \textbf{A marginal-value activation rule and greedy router}
  (Algorithm~\ref{alg:greedy}), together with two formally correct approximation
  guarantees appropriately scoped to the regimes they actually cover: a
  curvature-refined bound for a cardinality-constrained, monotone special case
  \citep{conforti1984submodular}, and a tight $1/2$-approximation (with an explicit
  correction for signed objectives) via double-greedy for the unconstrained,
  non-monotone case \citep{buchbinder2015tight} (Section~\ref{sec:marginal}).
  \item \textbf{A predictive, online Shapley-based routing mechanism}
  (Section~\ref{sec:shapley}) extending prior post-hoc Shapley credit assignment
  \citep{shapleycoop2025} to a routing-time role, together with a proved
  \emph{Shapley--submodularity sandwich bound} (Proposition~\ref{prop:sandwich},
  Corollary~\ref{cor:gap}) that bounds how far any marginal-value routing estimate
  can deviate from the fair Shapley credit, in terms of a single per-agent
  diminishing-returns quantity $\Gamma_i(x)$.
\end{enumerate}
We validate these three contributions with a controlled synthetic simulation
(Section~\ref{sec:simulation}) showing greedy marginal-value routing recovers
$99.5\%$ of brute-force-optimal utility while using roughly $25\%$ of the agents a
full-broadcast baseline would use, a sensitivity analysis
(Section~\ref{sec:sensitivity}) showing this result is robust to activation cost and
redundancy weight but degrades substantially when $V$'s submodularity is violated or
value estimates are noisy (exactly the two failure modes the theory predicts should
matter), and an empirical confirmation that the sandwich bound holds with zero
violations across all tested agents.

\paragraph{Extensions (not primary contributions).} We additionally sketch three
directions that are less developed and that we explicitly do not present as
established results: coalition stability via core, Nash, and pairwise-stability
concepts from hedonic game theory, where we identify why the classical
Shapley-in-the-core theorem does \emph{not} transfer to our submodular setting rather
than claim a stability theorem we cannot support (Section~\ref{sec:stability}); a
sketch of Vickrey--Clarke--Groves-style incentive alignment between individual agent
utility and system welfare (Section~\ref{sec:mechanism}); and a dynamic,
reinforcement-learning extension in which coalitions change during task execution
(Section~\ref{sec:dynamic}). These are flagged as future work throughout, not results
this paper establishes.

We emphasize at the outset that the simulation in Section~\ref{sec:simulation} is a
\emph{synthetic, controlled} instantiation of the framework -- it demonstrates that
the optimization problem we pose is well behaved and that a cheap greedy heuristic is
competitive with the true optimum under conditions the sensitivity analysis makes
explicit, not that the framework improves real LLM agent pipelines. Recent
neighboring work already reports experiments with real LLM agents at meaningful scale
-- e.g., \citet{coalitionstability2026} report $2{,}400$ episodes across GPT-4,
Claude-3, and Llama-3 -- and we hold this paper to a correspondingly honest standard:
we position it as a theoretical framework and controlled-simulation paper appropriate
for an arXiv preprint, not as a claim of validated improvement to real multi-agent
LLM systems, and we are explicit in Section~\ref{sec:limitations} about what remains
to be validated empirically.

\section{Related Work}
\label{sec:related}

\paragraph{Communication-efficient multi-agent LLM systems.} A growing line of work
prunes or learns the communication topology of multi-agent LLM systems rather than
using a fixed fully-connected graph. \citet{zhuge2024gptswarm} treat a multi-agent
system as an optimizable computation graph. \citet{zhang2025agentprune} prune
redundant communication edges based on message informativeness. Closely related to
our communication-cost term, recent work selects communication edges dynamically
based on trust, reliability, and disagreement among agents
\citep{trustaware2026}, which supports our premise that a fixed
``everyone-talks-to-everyone'' architecture is inefficient, but does not cast edge
selection as a cooperative-game-theoretic optimization jointly with coalition
membership, as we do in Section~\ref{sec:comm}.

\paragraph{LLM and multi-agent routing.} A separate line of work routes queries to
the cheapest sufficient model or agent configuration, including FrugalGPT
\citep{chen2023frugalgpt}, RouteLLM \citep{ong2024routellm}, and, in the multi-agent
setting, MasRouter \citep{yue2025masrouter}, xRouter \citep{qian2025xrouter}, and
state-aware agent selection \citep{wang2025optimalagent}. These systems are typically
trained end-to-end (e.g., with a learned controller or reinforcement learning) to
minimize cost subject to a quality target. Our framework is complementary: it
supplies an explicit, interpretable characteristic function and marginal-value
decision rule that such learned routers could be trained to approximate, and it
additionally addresses \emph{credit assignment} (Section~\ref{sec:shapley}) and
\emph{coalition stability} (Section~\ref{sec:stability}), which routing-only
formulations do not consider.

\paragraph{Cooperative game theory for LLM agents.} Shapley-Coop
\citep{shapleycoop2025} uses Shapley-value-based marginal contributions for
\emph{post-task} pricing and reward redistribution among self-interested LLM agents.
We build on this idea but use \emph{predicted}, pre-task and mid-task Shapley
estimates to drive the routing decision itself (Section~\ref{sec:shapley}), rather
than only settling accounts after the fact. Independently, \citet{coalitionstability2026}
study coalition formation among LLM agents through the lens of hedonic games,
analyze Nash-stable partitions under boundedly rational preferences, and -- notably
-- validate their framework with $2{,}400$ episodes across GPT-4, Claude-3, and
Llama-3, a substantially higher empirical bar than the synthetic simulation we
report in Section~\ref{sec:simulation}. Our contribution relative to this line is to
couple coalition formation with an explicit, task-conditioned communication-cost
term and a routing mechanism that acts online, and to connect the resulting
selection problem to submodular optimization (Section~\ref{sec:marginal}); we view
real-agent validation at comparable scale, following the protocol in
Section~\ref{sec:validation}, as the natural next step for this line of work rather
than something the present paper establishes.

\paragraph{Foundations.} We draw on the classical cooperative game theory of the
Shapley value \citep{shapley1953value} and the core \citep{gillies1959solutions}, on
hedonic games \citep{dreze1980hedonic,bogomolnaia2002hedonic}, on the
Vickrey--Clarke--Groves mechanism \citep{vickrey1961counterspeculation,clarke1971multipart,groves1973incentives},
and on submodular set-function optimization
\citep{nemhauser1978analysis,buchbinder2015tight}.

\section{Basic Formulation}
\label{sec:formulation}

Let $N = \{1, \dots, n\}$ be a set of candidate agents in an agentic AI system. Each
agent $i$ has a skill profile $S_i \subseteq \{\text{search}, \text{coding},
\text{planning}, \text{verification}, \dots\}$, or more generally a competence vector
over a latent skill space. For a task $x$, define a \emph{characteristic function}
$V(\cdot \mid x): 2^N \to \mathbb{R}$ that gives the expected net value produced when
the agents in a coalition $C \subseteq N$ collaborate on $x$, net of the costs that
depend on \emph{how} the coalition works together rather than on which specific
agents were invoked:
\begin{equation}
  V(C \mid x) = \mathbb{E}\Big[\, Q(C, x) - \lambda_2 L(C, x)
  - \lambda_3 R(C, x) - \lambda_4 D(C, x) \,\Big],
  \label{eq:v}
\end{equation}
where $Q$ is task-success quality, $L$ is latency, $R$ is the risk of hallucination,
failure, or unsafe action, and $D$ is duplicated work or redundant communication;
$\lambda_2,\lambda_3,\lambda_4 \geq 0$ are weights reflecting the deployment's cost
sensitivity. We deliberately keep the token/compute cost of \emph{invoking} an agent
\emph{out} of $V$ and instead attach it to the agent itself as a per-agent activation
price $c_i \geq 0$ (introduced in Section~\ref{sec:marginal}), so that the total
token/compute cost of a coalition is $K(C,x) := \sum_{i \in C} c_i$. This separation
is a deliberate accounting choice, not an additional assumption: $K$ and $c_i$ refer
to the same quantity, and keeping it out of $V$ prevents it from being subtracted
twice -- once inside $V$ and once again when the marginal-value activation rule
(Equation~\ref{eq:rule}) compares a marginal gain in $V$ against $c_i$. The system's
coalition-selection problem, over the true net utility of a coalition, is therefore
\begin{equation}
  C^*(x) = \arg\max_{C \subseteq N} \Big[V(C \mid x) - \!\!\sum_{i \in C} c_i\Big].
  \label{eq:argmax}
\end{equation}
Rather than activating every available skill, the system selects the coalition with
the highest net utility. Equation~\eqref{eq:v} nests the common special case
$Q(C,x)$ monotone non-decreasing and submodular in $C$ (adding an agent never hurts
quality, and does so with diminishing returns) as a natural regime, but we do not
assume monotonicity in general: redundant or conflicting agents can strictly reduce
$Q$, which is precisely why the grand coalition $C = N$ is not assumed optimal even
before token cost is considered.

\section{Communication as Part of the Optimization}
\label{sec:comm}

Given an active coalition $C$, let $G_C = (C, E_C)$ be the communication network
restricted to $C$, where $E_C \subseteq \{\{i,j\} : i,j \in C,\, i\neq j\}$ is an
\emph{undirected} set of communication channels: $\{i,j\} \in E_C$ means agents $i$
and $j$ are connected and can exchange messages, in one or both directions as the
task requires. We take edges to be undirected -- a connected pair shares a channel
whose cost $c_{ij}$ (below) is paid once regardless of how many messages flow in
each direction -- consistent with typical multi-agent debate and broadcast protocols,
where a connected pair's outputs are mutually visible; \emph{which} direction
information predominantly flows within an established channel (e.g., a coordinator
agent versus a peer) is a downstream policy choice we return to in
Section~\ref{sec:worked-example}, not a separate edge in $E_C$. Rather than fixing
$G_C$ a priori (e.g., fully connected, or a hand-designed pipeline), we treat it as
part of the optimization:
\begin{equation}
  \max_{C \subseteq N,\; G_C} \; U(C, G_C \mid x) \;-\!\! \sum_{\{i,j\} \in E_C} c_{ij},
  \label{eq:joint}
\end{equation}
where $U(C, G_C\mid x)$ is the net utility from Equation~\eqref{eq:argmax} (already
net of each agent's own activation cost) evaluated with the coalition communicating
over graph $G_C$, and $c_{ij}$ is the additional cost of the communication channel
between $i$ and $j$ (token cost of messages exchanged, added latency, and the risk
that one agent's error propagates into the other's context).
Problem~\eqref{eq:joint} operationalizes several design questions that are usually
answered by hand-coded architecture choices: which agents need to communicate; which
agent, if any, should act as a coordinator that most information routes through; when
an additional channel stops being worth its cost; whether two agents should
communicate directly or through an intermediary; and when a coalition should split
because communication has become too expensive relative to its benefit. Because $E_C$
ranges over a graph on $|C|$ nodes, exact solution of
Equation~\eqref{eq:joint} is combinatorial in general ($2^{\binom{|C|}{2}}$ candidate
graphs for a fixed $C$); Section~\ref{sec:marginal}
gives a tractable greedy procedure for the coalition-membership part of the problem,
and we treat joint edge optimization as an extension for future work
(Section~\ref{sec:limitations}).

\subsection{A worked example}
\label{sec:worked-example}

To make Equations~\eqref{eq:argmax} and~\eqref{eq:joint} concrete, consider a pool of
five candidate agents with complementary skills: a planner (P), a search/retrieval
agent (S), a coder (Co), a critic/verifier (Cr), and a security auditor (Se). Figure~\ref{fig:coalitions}
shows the coalition and communication graph a task-conditioned optimizer might select
for three different tasks, alongside the grand coalition for contrast.

\begin{figure}[t]
\centering
\begin{tikzpicture}[
  agent/.style={circle, draw, minimum size=7mm, font=\tiny, fill=blue!12},
  inactive/.style={circle, draw, minimum size=7mm, font=\tiny, fill=gray!8, draw=gray!45, text=gray!55},
  scale=0.68, every node/.style={transform shape}
]
\node[agent] (P) at (0,0.6) {P};
\node[inactive] (S) at (-0.9,-0.6) {S};
\node[inactive] (Co) at (0.9,-0.6) {Co};
\node[inactive] (Cr) at (-0.9,-1.8) {Cr};
\node[inactive] (Se) at (0.9,-1.8) {Se};
\node[below=0.4cm of Cr, xshift=0.45cm, font=\scriptsize, align=center]
  {(a) Simple writing task\\$C_1=\{\text{P}\}$, no communication needed};
\end{tikzpicture}
\hfill
\begin{tikzpicture}[
  agent/.style={circle, draw, minimum size=7mm, font=\tiny, fill=blue!12},
  inactive/.style={circle, draw, minimum size=7mm, font=\tiny, fill=gray!8, draw=gray!45, text=gray!55},
  scale=0.68, every node/.style={transform shape}
]
\node[agent] (P) at (0,0.6) {P};
\node[agent] (S) at (-0.9,-0.6) {S};
\node[inactive] (Co) at (0.9,-0.6) {Co};
\node[agent] (Cr) at (-0.9,-1.8) {Cr};
\node[inactive] (Se) at (0.9,-1.8) {Se};
\draw[-latex, thick] (P) -- (S);
\draw[-latex, thick] (S) -- (Cr);
\draw[-latex, thick, bend left=15] (Cr) to (P);
\node[below=0.4cm of Cr, xshift=0.45cm, font=\scriptsize, align=center]
  {(b) Factual research task\\$C_2=\{\text{P,S,Cr}\}$, planner-coordinated};
\end{tikzpicture}
\hfill
\begin{tikzpicture}[
  agent/.style={circle, draw, minimum size=7mm, font=\tiny, fill=blue!12},
  inactive/.style={circle, draw, minimum size=7mm, font=\tiny, fill=gray!8, draw=gray!45, text=gray!55},
  scale=0.68, every node/.style={transform shape}
]
\node[agent] (P) at (0,0.6) {P};
\node[inactive] (S) at (-0.9,-0.6) {S};
\node[agent] (Co) at (0.9,-0.6) {Co};
\node[agent] (Cr) at (-0.9,-1.8) {Cr};
\node[agent] (Se) at (0.9,-1.8) {Se};
\draw[-latex, thick] (P) -- (Co);
\draw[-latex, thick] (Co) -- (Cr);
\draw[-latex, thick] (Co) -- (Se);
\draw[-latex, thick, bend right=15] (Cr) to (P);
\node[below=0.4cm of Cr, xshift=0.45cm, font=\scriptsize, align=center]
  {(c) Software deployment task\\$C_3=\{\text{P,Co,Cr,Se}\}$, coder-coordinated};
\end{tikzpicture}
\hfill
\begin{tikzpicture}[
  agent/.style={circle, draw, minimum size=7mm, font=\tiny, fill=red!8},
  scale=0.68, every node/.style={transform shape}
]
\node[agent] (P) at (0,0.6) {P};
\node[agent] (S) at (-0.9,-0.6) {S};
\node[agent] (Co) at (0.9,-0.6) {Co};
\node[agent] (Cr) at (-0.9,-1.8) {Cr};
\node[agent] (Se) at (0.9,-1.8) {Se};
\foreach \a/\b in {P/S,P/Co,P/Cr,P/Se,S/Co,S/Cr,S/Se,Co/Cr,Co/Se,Cr/Se}{
  \draw[-, thick, red!45] (\a) -- (\b);
}
\node[below=0.4cm of Cr, xshift=0.45cm, font=\scriptsize, align=center]
  {(d) Grand coalition, fully connected\\$C=N$, $\binom{5}{2}=10$ edges for every task};
\end{tikzpicture}
\caption{A task-conditioned optimizer selects both \emph{which} agents to activate
and \emph{how} they communicate. Panels (a)--(c) show the coalition and
communication graph $(C, G_C)$ a solver of Equation~\eqref{eq:joint} might return
for three tasks of increasing complexity; edges are undirected communication
channels, drawn only where the marginal value of establishing the channel exceeds
its cost $c_{ij}$, with arrows annotating the predominant direction information
flows \emph{within} an established channel (e.g., toward a coordinating agent) rather
than indicating additional edges. Inactive agents (gray) are not
invoked at all, so neither their activation cost $c_i$ nor any channel
touching them is paid. Panel (d) shows the grand coalition every agent connected
to every other agent, regardless of task: this is the architecture our framework
argues against, since it pays for $\binom{5}{2}=10$ communication channels and $5$
activation costs even on the simple task in (a), where a single agent suffices.}
\label{fig:coalitions}
\end{figure}
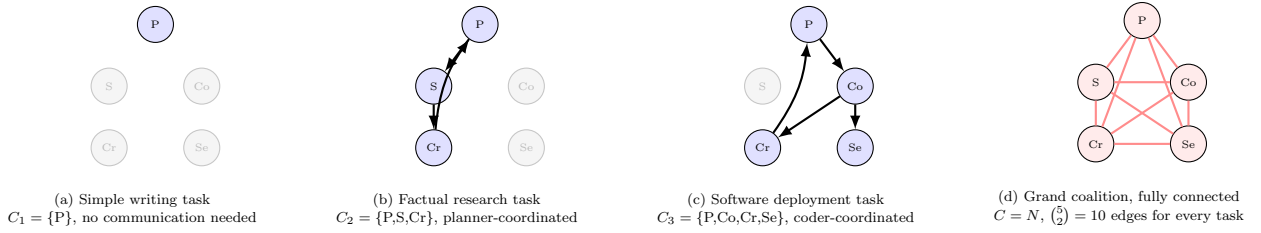

The three selected coalitions illustrate Section~\ref{sec:formulation}'s point that
more agents do not automatically create more net intelligence: $C_2$ and $C_3$ each
add exactly the agents whose skills are not already covered by the rest of the
coalition (search and verification for a research task; coding, verification, and
security review for a deployment task), and the communication graph in each case
routes information through a single coordinating agent (P in (b), Co in (c)) rather
than connecting every pair -- a sparse, task-specific topology of the kind
Section~\ref{sec:sandwich}'s per-agent diagnostic $\Gamma_i(x)$ and the marginal-value
routing rule of Equation~\eqref{eq:rule} are designed to produce automatically, rather
than requiring it to be hand-designed per task as in panels (a)--(c).

\section{Marginal Value and a Greedy Activation Rule}
\label{sec:marginal}

For a coalition $C$, the marginal contribution of agent $i$ is
\begin{equation}
  \Delta_i(C, x) = V(C \cup \{i\} \mid x) - V(C \mid x).
  \label{eq:marginal}
\end{equation}
This gives a natural activation rule: agent $i$ should be invoked if and only if
\begin{equation}
  \mathbb{E}[\Delta_i(C, x)] > c_i,
  \label{eq:rule}
\end{equation}
where $c_i$ is the cost of activating agent $i$ (e.g., its expected token and latency
footprint). For example, a verification agent should be added to a coalition that
already contains a planner and a coder only when its expected reduction in downstream
error exceeds its own token, latency, and API cost -- a more disciplined criterion
than routing every result through every available skill.

\begin{algorithm}[t]
\caption{Greedy marginal-value coalition router}
\label{alg:greedy}
\begin{algorithmic}[1]
\Require candidate agents $N$, task $x$, characteristic function $V(\cdot\mid x)$, activation costs $\{c_i\}$
\State $C \gets \emptyset$
\Repeat
  \State $a^* \gets \arg\max_{i \in N \setminus C} \big[V(C \cup \{i\}\mid x) - V(C\mid x)\big]$
  \State $g \gets V(C \cup \{a^*\}\mid x) - V(C\mid x)$
  \If{$g > c_{a^*}$}
    \State $C \gets C \cup \{a^*\}$
  \Else
    \State \textbf{break}
  \EndIf
\Until{$C = N$}
\State \Return $C$
\end{algorithmic}
\end{algorithm}

\subsection{Connection to submodular optimization}
\label{sec:submod-connection}

When $Q(C,x)$ is monotone submodular in $C$ -- a natural modeling choice, since
covering a task's skill requirements typically exhibits diminishing returns to
additional redundant agents -- and $L, R$ are modular (additive) in $C$ while $D$ is
supermodular (redundant work grows super-additively with coalition size, e.g.\ a
penalty on the number of overlapping pairs of agents), the objective $V(\cdot \mid x)$
in Equation~\eqref{eq:v} is itself submodular: subtracting a modular function
preserves submodularity, and subtracting a supermodular function (equivalently,
adding a submodular one) only adds curvature in the same direction. The full net
utility used for coalition selection, $U(C\mid x) := V(C\mid x) - \sum_{i\in C} c_i$
(Equation~\ref{eq:argmax}), remains submodular for the same reason, since
$\sum_{i\in C} c_i$ is modular. $U$ need not be monotone once costs are subtracted
(large coalitions can have lower value than smaller ones), so
Equation~\eqref{eq:argmax} is an instance of \emph{unconstrained submodular
maximization} on $U$, and we use $U$ rather than $V$ as the objective throughout the
remainder of the paper wherever a coalition is scored end-to-end (Algorithm~\ref{alg:greedy}
scores $U$, not $V$ alone).

\paragraph{Two equivalent forms of the activation rule.} Algorithm~\ref{alg:greedy}'s
stopping rule compares a marginal gain in $V$ against the per-agent cost $c_i$
(Equation~\ref{eq:rule}), which is convenient because $c_i$ is naturally a per-agent
quantity. This is exactly equivalent to requiring a positive marginal gain in the
full net utility $U$:
\begin{equation}
  \Delta_i V(C,x) > c_i \quad\Longleftrightarrow\quad \Delta_i U(C,x) > 0,
  \label{eq:equivalence}
\end{equation}
since $\Delta_i U(C,x) = U(C\cup\{i\}\mid x) - U(C\mid x) = \big[V(C\cup\{i\}\mid x) -
\!\!\sum_{j\in C\cup\{i\}}\!\! c_j\big] - \big[V(C\mid x) - \sum_{j\in C} c_j\big] =
\Delta_i V(C,x) - c_i$, which is positive exactly when $\Delta_i V(C,x) > c_i$. We use
this equivalence to fix notation for the remainder of the paper and avoid mixing the
two forms: \textbf{Algorithm~\ref{alg:greedy}} is stated in the $V$-plus-$c_i$ form
(Equation~\ref{eq:rule}), since that is the natural form for a per-agent streaming
decision. \textbf{Algorithm~\ref{alg:doublegreedy}, Proposition~\ref{prop:doublegreedy},
and the Shapley analysis of Section~\ref{sec:shapley}} are instead stated directly on
$U$, the full net utility from Equation~\eqref{eq:argmax}: this is the objective a
system actually cares about maximizing end-to-end, and it is the game whose Shapley
value gives agents fair credit net of their own cost (an agent that produces a lot of
quality but is expensive should not receive full credit for the quality alone). We do
not use $V$ to stand in for $U$ anywhere below; each proposition states explicitly
which of the two objectives it concerns.

\subsection{Approximation guarantees}
\label{sec:approx-guarantees}

We give two formal guarantees for greedy coalition selection, covering the two
regimes a deployment may face: a hard budget on the number of active agents
(cardinality-constrained), and no explicit budget, with cost handled entirely through
$V$ (unconstrained).

\paragraph{Cardinality-constrained regime.} We state this regime as a special case for
context, and to introduce curvature as a diagnostic quantity; it does \emph{not}
directly govern Algorithm~\ref{alg:greedy}, our main router, which is unconstrained
rather than cardinality-constrained, maximizes the signed, non-monotone $U(C\mid x)$
rather than the monotone quality term $Q(C,x)$ alone, and uses a per-agent stopping
threshold rather than a fixed budget $k$. Suppose instead the system imposes a budget
$|C| \leq k$ (e.g., a hard cap on concurrently active agents) and, within that budget,
selects agents to greedily maximize the quality term $Q(C,x)$, monotone submodular by
assumption. The classical guarantee for this setting is $(1-1/e) \approx 63\%$ of the
optimal value under a cardinality constraint \citep{nemhauser1978analysis}, but this
bound is loose whenever $Q$ is close to modular (i.e., agents' skills are close to
non-overlapping, so there is little redundancy to lose from greedy's sub-optimality).
This looseness is captured by the \emph{curvature} of $Q$,
\begin{equation}
  \alpha_Q = 1 - \min_{i \in N} \frac{Q(N) - Q(N \setminus \{i\}) }{Q(\{i\}) - Q(\emptyset)} \; \in [0,1],
  \label{eq:curvature}
\end{equation}
which measures how much an agent's marginal contribution can shrink between being
added first versus added last ($\alpha_Q = 0$ means no diminishing returns at all,
i.e.\ $Q$ is modular; $\alpha_Q = 1$ means some agent's marginal contribution can
fall to zero once other agents are present).

\begin{proposition}[Curvature-refined greedy bound]
\label{prop:curvature}
If $Q(\cdot, x)$ is monotone submodular with curvature $\alpha_Q$, greedy selection
under a cardinality constraint $|C|\le k$ achieves
\begin{equation}
  Q(C_{\mathrm{greedy}}, x) \;\geq\; \frac{1}{\alpha_Q}\Big(1 - e^{-\alpha_Q}\Big)\, Q(C^*_k, x),
  \label{eq:curvature-bound}
\end{equation}
where $C^*_k$ is the optimal coalition of size $\le k$
\citep{conforti1984submodular}. The bound recovers the classical $(1-1/e)$ guarantee
at $\alpha_Q = 1$ and improves continuously to a guarantee of $1$ (greedy is exactly
optimal) as $\alpha_Q \to 0$.
\end{proposition}

Equation~\eqref{eq:curvature-bound} is a diagnostic a system operator could compute
\emph{in the constrained, monotone special case}: estimate $\alpha_Q$ empirically from
historical coalition outcomes (the smallest observed ratio of ``last-added'' to
``first-added'' marginal value for any agent), and read off an instance-specific lower
bound rather than relying on the worst-case $(1-1/e)$ figure. We report the empirical
performance of Algorithm~\ref{alg:greedy} itself, on the unconstrained, signed $U$,
separately in Section~\ref{sec:simulation}; Proposition~\ref{prop:curvature} should be
read as motivating intuition for why low skill overlap is favorable to greedy
selection in general, not as a bound that applies to that specific result.

\paragraph{Unconstrained regime.} Algorithm~\ref{alg:greedy} does not impose a
cardinality constraint; it stops adding agents once no marginal gain clears its
activation cost, which is the setting we use throughout the paper. Here the relevant
guarantee is for \emph{unconstrained} submodular maximization of the full net utility
$U$, which can be negative and non-monotone (exactly our setting once activation and
other costs are subtracted). Algorithm~\ref{alg:greedy} does not attain the guarantee
below, because it only adds agents and never removes them; we present the guarantee
for a variant, Algorithm~\ref{alg:doublegreedy}, which does.

\begin{algorithm}[t]
\caption{Randomized double-greedy (unconstrained submodular maximization) \citep{buchbinder2015tight}}
\label{alg:doublegreedy}
\begin{algorithmic}[1]
\Require candidate agents $N = \{1,\dots,n\}$, task $x$, $U(\cdot\mid x)$
\State $X_0 \gets \emptyset$, $Y_0 \gets N$
\For{$i = 1$ to $n$}
  \State $a_i \gets U(X_{i-1}\cup\{i\}\mid x) - U(X_{i-1}\mid x)$ \Comment{gain from adding $i$}
  \State $b_i \gets U(Y_{i-1}\setminus\{i\}\mid x) - U(Y_{i-1}\mid x)$ \Comment{gain from removing $i$}
  \State $a_i' \gets \max(a_i, 0)$, $b_i' \gets \max(b_i, 0)$
  \State with probability $\frac{a_i'}{a_i' + b_i'}$ (or $\tfrac12$ if $a_i'=b_i'=0$): $X_i \gets X_{i-1}\cup\{i\}$, $Y_i \gets Y_{i-1}$
  \State otherwise: $X_i \gets X_{i-1}$, $Y_i \gets Y_{i-1}\setminus\{i\}$
\EndFor
\State \Return $X_n$ ($= Y_n$)
\end{algorithmic}
\end{algorithm}

Algorithm~\ref{alg:doublegreedy} takes $U$ (not $V$) as input: since $U$ is the
objective the system actually wants to maximize end-to-end (Equation~\ref{eq:argmax}),
and since $U$ is submodular whenever $V$ is (shown above, as $U$ is $V$ minus a
modular per-agent cost term), the guarantee below is stated directly for $U$.

\begin{proposition}[Tight unconstrained bound for non-negative submodular functions \citep{buchbinder2015tight}]
\label{prop:doublegreedy}
If $U(\cdot\mid x): 2^N \to \mathbb{R}_{\geq 0}$ is submodular and non-negative (not
necessarily monotone), Algorithm~\ref{alg:doublegreedy} returns a coalition $C$ with
$\mathbb{E}[U(C\mid x)] \geq \tfrac{1}{2}\, U(C^*\mid x)$, where $C^* = \arg\max_{C\subseteq
N} U(C\mid x)$. In the value-oracle model (the algorithm may query $U$ at any coalition
but has no further structural knowledge of it), no algorithm making a polynomial number
of queries can guarantee a better ratio in the worst case \citep{buchbinder2015tight,feige2011maximizing}.
\end{proposition}

The non-negativity requirement matters: our $U(\cdot\mid x)$ (Equation~\ref{eq:argmax})
can be negative, since a large, redundant coalition can have quality gains outweighed by
latency, risk, redundancy, and activation cost. The guarantee above does not apply
directly to a signed objective, but it extends to one by a simple shift argument, which
we state as an explicitly limited analytical device rather than a practical algorithm in
its own right (see the caveats following the corollary).

\begin{corollary}[Extension to signed $U$]
\label{cor:shift}
Let $m = \min_{C\subseteq N} U(C\mid x)$ and define the shifted objective
$\tilde U(C\mid x) := U(C\mid x) - m \geq 0$. Since adding a constant to a set function
preserves submodularity, $\tilde U$ is submodular and non-negative, so
Proposition~\ref{prop:doublegreedy} applies to $\tilde U$: running
Algorithm~\ref{alg:doublegreedy} on $\tilde U$ returns $C$ with
$\mathbb{E}[\tilde U(C\mid x)] \geq \tfrac12\, \tilde U(C^*\mid x)$, which translates
back to
\begin{equation}
  \mathbb{E}[U(C\mid x)] \;\geq\; \tfrac12\, U(C^*\mid x) + \tfrac12\, m.
  \label{eq:shift-bound}
\end{equation}
\end{corollary}
\begin{proof}
Substitute $\tilde U(C\mid x) = U(C\mid x) - m$ into
$\mathbb{E}[\tilde U(C\mid x)] \geq \tfrac12 \tilde U(C^*\mid x)$ and rearrange:
$\mathbb{E}[U(C\mid x)] - m \geq \tfrac12(U(C^*\mid x) - m)$, i.e.\
$\mathbb{E}[U(C\mid x)] \geq \tfrac12 U(C^*\mid x) - \tfrac12 m + m = \tfrac12
U(C^*\mid x) + \tfrac12 m$.
\end{proof}

This shift is primarily an analytical device rather than a practical algorithm in its
own right, for two reasons that should temper how it is read: computing the exact
minimum $m = \min_{C\subseteq N} U(C\mid x)$ is itself a combinatorial search over
$2^n$ coalitions with no known efficient algorithm in general (estimating $m$ instead
would weaken the guarantee further); and when $m$ is substantially negative -- plausible
whenever redundancy or risk costs can be large relative to quality -- the resulting
additive term $\tfrac12 m$ can make Equation~\eqref{eq:shift-bound} numerically
uninformative (e.g.\ vacuous if $\tfrac12 U(C^*\mid x) + \tfrac12 m \leq 0$, which any
non-negative utility trivially satisfies). We report the shift because it is the
mathematically correct extension of Proposition~\ref{prop:doublegreedy} to a signed
objective, not because it yields a bound we expect to be tight or actionable in
practice.

Because $m \leq 0$ typically (the empty coalition, or a coalition dominated by
redundancy cost, achieves the minimum), Equation~\eqref{eq:shift-bound} is a strictly
weaker guarantee than the clean $\tfrac12$-approximation that holds for non-negative
functions; we report it in this explicit, weaker form rather than applying the
non-negative-case bound to a signed objective without adjustment.

Algorithm~\ref{alg:doublegreedy} is a strict theoretical improvement over
Algorithm~\ref{alg:greedy} in the worst case, since add-only greedy has no known
constant-factor guarantee once $U$ is allowed to be non-monotone and signed. In practice
we use Algorithm~\ref{alg:greedy} because it is simpler to implement as a stateful,
streaming router (it only ever grows the active coalition, which maps naturally onto
sequentially deciding whether to contact the next agent). Section~\ref{sec:simulation}
reports how close Algorithm~\ref{alg:greedy} comes to the brute-force optimum on our
synthetic instances \emph{as an empirical finding}; this should not be read as evidence
that Algorithm~\ref{alg:greedy} inherits either Proposition~\ref{prop:curvature} or
Proposition~\ref{prop:doublegreedy}, since it satisfies the preconditions of neither
(it is unconstrained rather than cardinality-constrained, and it only adds agents rather
than running the randomized add/remove procedure double-greedy requires). We view
Algorithm~\ref{alg:doublegreedy} as the theoretically principled fallback when the
marginal-value routing rule of Equation~\eqref{eq:rule} is suspected to be a poor
approximation -- e.g., when $\hat\alpha_Q$, estimated from logged data, is close to $1$.

\begin{remark}[Robustness to approximate submodularity]
\label{rem:robustness}
Real agent-quality functions $Q$ need not be exactly submodular; conflicting or
miscalibrated agents can occasionally violate diminishing returns locally. The
guarantees above degrade gracefully rather than catastrophically in this case: if $Q$
is only \emph{approximately} submodular with submodularity ratio $\gamma \in (0,1]$
(the largest $\gamma$ such that $Q$ satisfies a $\gamma$-relaxed diminishing-returns
condition), the cardinality-constrained greedy guarantee degrades to
$(1-e^{-\gamma})$ \citep{das2011submodular}, recovering
Proposition~\ref{prop:curvature} at $\gamma=1$. We do not estimate $\gamma$ in this
paper, but flag it as the right empirical quantity for a follow-up study to report
alongside $\hat\alpha_Q$.
\end{remark}

\section{Online Credit Assignment with Shapley Values}
\label{sec:shapley}

After a coalition completes a task, value must be attributed to individual agents.
We compute the Shapley value of the full net-utility game $U$ (not $V$), so that an
agent's credit already accounts for its own activation cost -- an agent that produces
a lot of quality but is expensive to run should not receive full credit for the
quality alone. The Shapley value gives agent $i$ credit
\begin{equation}
  \phi_i(U) = \sum_{C \subseteq N \setminus \{i\}} \frac{|C|!\,(n-|C|-1)!}{n!}
  \Big[U(C \cup \{i\}) - U(C)\Big],
  \label{eq:shapley}
\end{equation}
which is the unique allocation satisfying efficiency ($\sum_i \phi_i = U(N)$),
symmetry, the null-player property, and additivity \citep{shapley1953value}. Prior
work uses this decomposition \emph{post hoc}, to price contributions and redistribute
reward among self-interested LLM agents after a task completes
\citep{shapleycoop2025}. Our proposal is to move Shapley estimation earlier: use a
predicted marginal contribution $\hat\phi_{i,t}(U)$ -- an online estimate of
$\phi_i(U)$, and therefore \emph{already net of agent $i$'s activation cost}, since
$U$ is the cost-net game (Equation~\ref{eq:argmax}) -- estimated from features
available at routing time (task embedding, agent track record, partial coalition
state), to decide whether to contact agent $i$ at stage $t$:
\begin{equation}
  P(i \text{ contacted at time } t) = \sigma\!\big(\hat\phi_{i,t}(U)\big),
  \label{eq:online-shapley}
\end{equation}
where $\sigma$ is a monotone squashing function (e.g., logistic) centered so that
$\sigma(0) = \tfrac12$: since $\hat\phi_{i,t}(U)$ already nets out $c_i$, no further
cost term is subtracted here -- doing so would subtract agent $i$'s activation cost a
second time, once inside the estimate of $\phi_i(U)$ and once again explicitly, which
is exactly the accounting error Equation~\eqref{eq:equivalence} was introduced to
avoid elsewhere in the paper. Realized
Shapley values from completed tasks, computed exactly for small $|C|$ or approximated
via Monte Carlo permutation sampling for larger coalitions, then supervise the
predictor $\hat\phi_{i,t}(U)$, closing the loop between post-hoc credit assignment and
future routing probability. This also gives a natural signal for (i) adjusting
routing priors, (ii) allocating token/compute budgets across agents, (iii) pricing
contributions when agents belong to different providers, and (iv) identifying agents
that consistently free-ride on others' contributions.

\subsection{A theoretical bound on agent contribution: the Shapley--submodularity sandwich}
\label{sec:sandwich}

Equation~\eqref{eq:online-shapley} routes on an \emph{estimate} $\hat\phi_{i,t}(U)$ of
the true Shapley value $\phi_i(U)$, while Algorithm~\ref{alg:greedy} routes on the raw marginal contribution
$\Delta_i(C,x) = \Delta_i V(C,x)$ to whatever partial coalition $C$ has been built so
far (Equation~\eqref{eq:marginal}), compared against $c_i$. By the equivalence
established in Equation~\eqref{eq:equivalence}, Algorithm~\ref{alg:greedy}'s decision
is equivalent to asking whether the corresponding marginal contribution to the full
net utility, $\Delta_i U(C,x) := U(C\cup\{i\}\mid x) - U(C\mid x) = \Delta_i V(C,x) -
c_i$, is positive. It is this $U$-marginal, not the raw $V$-marginal, that we compare
against the Shapley value $\phi_i(U)$ below, since both live in the same game.

\begin{proposition}[Shapley--submodularity sandwich]
\label{prop:sandwich}
If $U(\cdot \mid x)$ is submodular (diminishing marginal returns: for all $A
\subseteq B \subseteq N\setminus\{i\}$, $U(A\cup\{i\}) - U(A) \geq U(B\cup\{i\}) -
U(B)$ -- which holds whenever $V$ is submodular, since $U$ differs from $V$ only by a
modular cost term, Section~\ref{sec:submod-connection}), then for every agent $i$ and
every coalition $C$ with $i \notin C$,
\begin{equation}
  U(N\mid x) - U(N\setminus\{i\}\mid x) \;\leq\; \Delta_i U(C, x) \;\leq\; U(\{i\}\mid x) - U(\emptyset \mid x),
  \label{eq:sandwich}
\end{equation}
and in particular the Shapley value itself satisfies the same sandwich,
\begin{equation}
  U(N\mid x) - U(N\setminus\{i\}\mid x) \;\leq\; \phi_i(U) \;\leq\; U(\{i\}\mid x) - U(\emptyset\mid x).
  \label{eq:sandwich-shapley}
\end{equation}
\end{proposition}

\begin{proof}
The Shapley value can be written as an expectation over uniformly random orderings
$\pi$ of $N$: $\phi_i(U) = \mathbb{E}_\pi\big[U(P_i^\pi \cup \{i\}) - U(P_i^\pi)\big]$,
where $P_i^\pi$ is the set of agents preceding $i$ in $\pi$. For every ordering,
$\emptyset \subseteq P_i^\pi \subseteq N\setminus\{i\}$, so by the diminishing-returns
definition of submodularity applied at $A = P_i^\pi$ and the two extremes $A=\emptyset$
and $A = N\setminus\{i\}$,
\[
  U(N\mid x) - U(N\setminus\{i\}\mid x) \;\leq\; U(P_i^\pi \cup \{i\}\mid x) - U(P_i^\pi\mid x) \;\leq\; U(\{i\}\mid x) - U(\emptyset\mid x)
\]
holds for \emph{every} $\pi$, hence for $\Delta_i U(C,x)$ directly (taking $C = P_i^\pi$
for any single ordering), and survives taking the expectation over $\pi$ that defines
$\phi_i(U)$, giving Equation~\eqref{eq:sandwich-shapley}.
\end{proof}

\begin{corollary}[Bounded routing gap]
\label{cor:gap}
Define agent $i$'s \emph{diminishing-returns gap} as
\begin{equation}
  \Gamma_i(x) \;=\; \big[U(\{i\}\mid x) - U(\emptyset \mid x)\big] \;-\; \big[U(N\mid x) - U(N\setminus\{i\}\mid x)\big] \;\geq 0.
  \label{eq:gap}
\end{equation}
Then for any coalition $C$ built by Algorithm~\ref{alg:greedy} at the point agent $i$
is considered, $\big|\Delta_i U(C,x) - \phi_i(U)\big| \leq \Gamma_i(x)$. By the
equivalence in Equation~\eqref{eq:equivalence}, $\Delta_i U(C,x) = \Delta_i(C,x) -
c_i$ where $\Delta_i(C,x)$ is exactly the quantity Algorithm~\ref{alg:greedy} already
computes, so this bounds a simple, known offset ($\Delta_i(C,x) - c_i$) of Algorithm~\ref{alg:greedy}'s
own decision signal against the fair Shapley credit, in terms of agent $i$'s own
diminishing-returns gap, independent of when in the coalition-building process $i$
is evaluated.
\end{corollary}

Corollary~\ref{cor:gap} establishes that a specific marginal-contribution quantity is
close to Shapley-fair credit; it does \emph{not}, by itself, establish that
Algorithm~\ref{alg:greedy} selects the correct coalition, incurs small regret, or
routes optimally -- an agent can have a marginal value close to its Shapley value
while the router still makes the wrong decision, for instance because of noise in the
value estimate or because the agent sits near the activation threshold $c_i$ (see the
noise sensitivity sweep in Section~\ref{sec:sensitivity}, where exact-match rate with
the optimum degrades far faster than utility does under estimation noise, for exactly
this reason). With that caveat, Corollary~\ref{cor:gap} still gives a precise,
testable diagnostic for when the cheap greedy activation rule
(Equation~\ref{eq:rule}) is a safe proxy for Shapley-fair routing
(Equation~\ref{eq:online-shapley}), and when the more expensive online Shapley
estimator is worth its cost: agents with small $\Gamma_i(x)$ (little overlap with
other agents' skills, so their contribution barely depends on who else is present)
have a raw marginal contribution that is a reliable proxy for their Shapley credit,
making Algorithm~\ref{alg:greedy} more defensible \emph{when its value estimates are
also accurate}; agents with large $\Gamma_i(x)$ (highly redundant with other agents,
so their marginal contribution swings sharply depending on coalition composition --
e.g., a second verifier agent whose value depends heavily on whether a first verifier
is already active) need the full online Shapley machinery of
Section~\ref{sec:shapley} to be routed fairly. Summed across agents, $\sum_{i \in N}
\Gamma_i(x)$ is a single diagnostic a system operator can log per task family: a
small sum indicates the agent pool is close to skill-orthogonal (low curvature, in
the sense of Equation~\ref{eq:curvature}) and credit-approximation error from using
raw marginal value is small; a large sum indicates substantial skill overlap and
motivates paying for Shapley-based routing or the double-greedy fallback of
Algorithm~\ref{alg:doublegreedy}. Together with the curvature bound of
Proposition~\ref{prop:curvature} (Section~\ref{sec:approx-guarantees}), this connects
the paper's two main theoretical tools -- submodular optimization and the Shapley
value -- into a single diagnostic quantity ($\alpha_Q$, or its per-agent analogue
$\Gamma_i$) that bounds \emph{credit-assignment} accuracy; it is supportive evidence
for, but not a proof of, good coalition-selection quality, which is instead
addressed (separately, and only in restricted regimes) by
Propositions~\ref{prop:curvature} and~\ref{prop:doublegreedy}.

\section{Coalition Stability}
\label{sec:stability}

Maximizing $U(C \mid x)$ for a single task does not guarantee that a multi-task
\emph{system} of coalitions is stable: individual agents, or providers operating
agents, may prefer to leave one coalition for another. We adopt three standard
stability notions from hedonic game theory
\citep{dreze1980hedonic,bogomolnaia2002hedonic}.

\begin{definition}[Core stability]
An allocation $y_1,\dots,y_n$ with $\sum_{i \in N} y_i = U(N)$ is in the \emph{core}
if $\sum_{i \in C} y_i \geq U(C)$ for every $C \subseteq N$: no subgroup can produce
more value by leaving the grand coalition and splitting off on its own
\citep{gillies1959solutions}.
\end{definition}

\begin{definition}[Nash-stable partition]
A partition $\pi$ of $N$ into coalitions is \emph{Nash-stable} if no agent strictly
prefers moving from its current coalition to another coalition in $\pi$ (including
the empty coalition), under each agent's individual utility function.
\end{definition}

\begin{definition}[Pairwise stability]
A communication graph $G$ is \emph{pairwise stable} if no pair of agents benefits,
jointly, from adding or removing the edge between them.
\end{definition}

\begin{remark}[Why the classical convex-game result does not transfer]
\label{rem:core-mismatch}
The classical result that the Shapley value lies in the core applies to \emph{convex}
games -- games whose characteristic function is \emph{supermodular}, i.e.\ has
\emph{increasing} marginal returns: $U(A\cup\{i\}) - U(A) \leq U(B\cup\{i\}) -
U(B)$ for $A \subseteq B$ \citep{shapley1971cores}. Our framework's central modeling
assumption (Section~\ref{sec:marginal}) is the opposite: $U$ is \emph{submodular},
with \emph{diminishing} marginal returns, precisely because redundant agents
contribute less as the coalition grows. Submodularity and supermodularity are, in
this sense, opposing conditions, and a game that is (non-trivially) both is
essentially modular. Consequently, the classical Shapley-in-the-core theorem does
\emph{not} apply to the coalition-formation game studied in this paper, and we do
not claim it does; a restricted superadditivity condition on skill-complementary
sub-coalitions is not sufficient to establish the supermodularity the theorem
actually requires. The only case in which core stability is immediate is the
degenerate one in which $U$ is exactly modular -- i.e., agents' skills are fully
non-overlapping, $D \equiv 0$, and activation costs are the only cost term, so
$U(C) = \sum_{i\in C} U(\{i\})$ for every $C$ -- in which case the
unique core allocation $y_i = U(\{i\})$ coincides trivially with the Shapley value.
Beyond this degenerate case, whether (and under what conditions) our submodular
coalition-formation game admits a non-empty core, and how that relates to the
Nash-stable partitions studied for LLM agent coalitions under hedonic game theory
\citep{coalitionstability2026}, is an open question that we leave to future work
rather than a result we claim here.
\end{remark}

\section{Local Utility, System Welfare, and Incentive Alignment}
\label{sec:mechanism}

Agents (or the providers operating them) may have private utilities that diverge from
system-level welfare. A representative local utility is
\begin{equation}
  u_i = \alpha_i \cdot \text{credit}_i - \beta_i \cdot \text{computation}_i -
  \gamma_i \cdot \text{failure\_exposure}_i,
  \label{eq:local-utility}
\end{equation}
while the platform's objective is
\begin{equation}
  W = \text{task quality} - \lambda\cdot\text{cost} - \mu\cdot\text{latency} - \rho\cdot\text{risk}.
  \label{eq:welfare}
\end{equation}
Left unmanaged, a self-interested agent may avoid difficult tasks, exaggerate
confidence to be selected, send unnecessary messages to appear useful, over-delegate,
withhold useful information, or form coalitions that benefit its members at the
system's expense. In an idealized setting with quasilinear utilities, exact welfare
maximization, and correctly specified Clarke pivot payments, a
Vickrey--Clarke--Groves mechanism -- paying agent $i$ its marginal contribution to
welfare net of the externality it imposes on the rest of the coalition
\citep{vickrey1961counterspeculation,clarke1971multipart,groves1973incentives} -- would
make truthful valuation reporting a dominant strategy, aligning $u_i$ with $W$ at the
margin. Whether an analogous result can be obtained using \emph{noisy, estimated}
Shapley-based contribution values $\hat\phi_i$ in a real LLM-agent setting -- where
welfare is not exactly maximized, valuations are not quasilinear, and payments would
be based on an approximate credit-assignment signal rather than exact marginal
welfare -- remains open. We do not claim the idealized guarantee transfers; VCG
mechanisms are also generally not budget-balanced, which is a separate obstacle to
direct implementation. We view the idealized mechanism only as the right
\emph{target} to approximate, and a natural direction for follow-up work
combining Section~\ref{sec:shapley}'s online Shapley estimator with an explicit
payment or reward-shaping rule.

\section{Extension: Dynamic Coalition Formation (Future Work)}
\label{sec:dynamic}

We flag, but do not develop or validate, one further extension. Coalitions plausibly
should change during task execution rather than being fixed at the outset. Let the
state at time $t$ be $s_t = (\text{task progress}, \text{agent
outputs}, \text{confidence}, \text{remaining budget}, \text{disagreement})$, and let
the action be $a_t = (\text{active coalition}, \text{communication links},
\text{message content}, \text{next agent})$. A dynamic objective is
\begin{equation}
  \max_{\pi} \; \mathbb{E}_\pi\!\left[\sum_{t=0}^{T} \gamma^t (r_t - c_t)\right],
  \label{eq:dynamic}
\end{equation}
learned with multi-agent reinforcement learning, where the static cooperative game of
Sections~\ref{sec:formulation}--\ref{sec:mechanism} supplies the per-step reward
shaping and credit assignment ($r_t$ derived from $U(\cdot\mid x)$, already net of
per-agent activation cost, and $c_t$ from the communication costs). We flag this extension as substantially harder
than static coalition selection -- it compounds the credit-assignment problem across
a sequence of coalition changes -- and we treat it as a natural but separate line of
future work (Section~\ref{sec:limitations}) rather than a claim validated in this
paper.

\section{Illustrative Simulation}
\label{sec:simulation}

To sanity-check the framework before committing to a full empirical study with real
LLM agents, we instantiate Equation~\eqref{eq:v} in a controlled synthetic setting.

\paragraph{Setup.} We generate $N = 8$ candidate agents, each with a competence
vector over $K = 6$ latent skill dimensions in $[0,1]$. Each agent has one primary
skill drawn with competence in $[0.85, 1.0]$, and, with probability $0.5$, a
secondary skill with competence in $[0.5, 0.8]$; competence on all other dimensions is
low ($[0.05, 0.35]$), reflecting the planner/coder/search/critic/security archetype of
Section~\ref{sec:formulation}. We generate $500$ synthetic tasks, each with a random
skill-demand vector drawn from a Dirichlet distribution over the $K$ dimensions.
Coalition quality is a saturating function of how well the coalition's best-per-skill
competence covers the task's demand,
$Q(C, x) = 1 - \exp(-2.2 \cdot \text{coverage}(C,x))$, where
$\text{coverage}(C, x) = \sum_k \max_{i \in C} \text{skill}_i(k)\cdot\text{demand}(k)
/ \sum_k \text{demand}(k)$. Consistent with Equation~\eqref{eq:v}, $V(C,x) = Q(C,x) -
\lambda_4 D(C,x)$ excludes token/compute cost entirely: $\lambda_2 L$ and $\lambda_3 R$
are set to zero in this simplified instantiation, and the redundancy term grows
quadratically in coalition size, normalized by the number of agent pairs
($\lambda_4 D(C,x) = 0.16 \cdot \binom{|C|}{2} / \binom{N}{2}$). The \emph{only} place
token/compute cost enters is the per-agent activation price $c_i = 0.05$ (uniform
across agents), applied exactly once via $U(C\mid x) = V(C\mid x) - \sum_{i\in C} c_i$
(Equation~\ref{eq:argmax}); no other cost term duplicates it.

\paragraph{Policies compared.} (i) \emph{Grand coalition}: activate all $N$ agents.
(ii) \emph{Random coalition}: activate a uniformly random subset. (iii)
\emph{Greedy marginal-value router} (Algorithm~\ref{alg:greedy}), our proposal, with
activation cost $c_i = 0.05$ applied uniformly (Section~\ref{sec:formulation}). (iv)
\emph{Brute-force optimum}: exhaustive search over
all $2^8 = 256$ coalitions (tractable at this scale, and used only as an upper-bound
reference -- not a proposal for deployment, since it does not scale to realistic
agent-pool sizes). All four policies are scored on the same objective, the full net
utility $U(C\mid x) = V(C\mid x) - \sum_{i\in C} c_i$ from Equation~\eqref{eq:argmax},
so that comparisons across policies are on equal footing.

\begin{table}[t]
\centering
\caption{Mean utility and coalition size across $500$ synthetic tasks. Greedy
marginal-value routing recovers $99.5\%$ of brute-force-optimal utility while
activating $78\%$ fewer agents, on average, than the grand coalition.}
\label{tab:results}
\begin{tabular}{lccc}
\toprule
Policy & Mean utility $U(C\mid x)$ & Mean $|C|$ & \% of optimal utility \\
\midrule
Grand coalition (broadcast all) & $0.262$ & $8.00$ & $38.8\%$ \\
Random coalition & $0.453$ & $4.46$ & $67.1\%$ \\
\textbf{Greedy marginal-value (proposed)} & $\mathbf{0.671}$ & $\mathbf{1.96}$ & $\mathbf{99.5\%}$ \\
Brute-force optimum (upper bound) & $0.675$ & $1.97$ & $100.0\%$ \\
\bottomrule
\end{tabular}
\end{table}

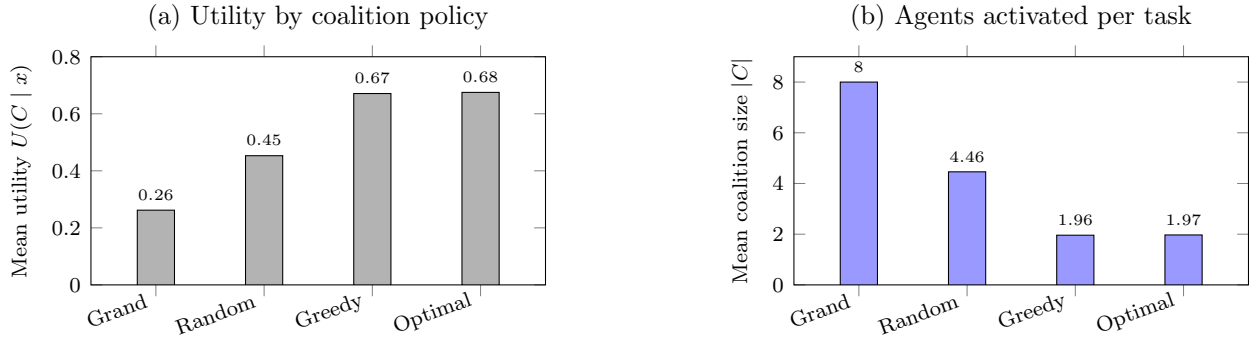
\begin{figure}[t]
\centering
\begin{tikzpicture}
\begin{axis}[
  ybar, bar width=14pt, width=0.46\textwidth, height=4.6cm,
  symbolic x coords={Grand,Random,Greedy,Optimal},
  xtick=data, x tick label style={font=\scriptsize, rotate=20, anchor=east},
  ylabel={Mean utility $U(C\mid x)$}, ylabel style={font=\scriptsize},
  ymin=0, ymax=0.8, title={(a) Utility by coalition policy}, title style={font=\small},
  nodes near coords, nodes near coords style={font=\tiny},
  enlarge x limits=0.2, tick label style={font=\scriptsize}
]
\addplot[fill=gray!60] coordinates {(Grand,0.262) (Random,0.453) (Greedy,0.671) (Optimal,0.675)};
\end{axis}
\end{tikzpicture}
\hfill
\begin{tikzpicture}
\begin{axis}[
  ybar, bar width=14pt, width=0.46\textwidth, height=4.6cm,
  symbolic x coords={Grand,Random,Greedy,Optimal},
  xtick=data, x tick label style={font=\scriptsize, rotate=20, anchor=east},
  ylabel={Mean coalition size $|C|$}, ylabel style={font=\scriptsize},
  ymin=0, ymax=9, title={(b) Agents activated per task}, title style={font=\small},
  nodes near coords, nodes near coords style={font=\tiny},
  enlarge x limits=0.2, tick label style={font=\scriptsize}
]
\addplot[fill=blue!40] coordinates {(Grand,8.00) (Random,4.46) (Greedy,1.96) (Optimal,1.97)};
\end{axis}
\end{tikzpicture}
\caption{(a) Mean task utility and (b) mean number of agents activated, by coalition
policy, over $500$ synthetic tasks. The greedy marginal-value router
(Algorithm~\ref{alg:greedy}) nearly matches the brute-force optimum in utility while
using far fewer agents than full broadcast.}
\label{fig:results}
\end{figure}

\paragraph{Results.} Table~\ref{tab:results} and Figure~\ref{fig:results} summarize
the outcome. The grand coalition performs worst despite activating every agent: at
this cost setting, redundant agents actively subtract value through the compute and
redundancy penalties, so full broadcast captures only $38.8\%$ of the utility a
well-chosen coalition achieves. The greedy marginal-value router
(Algorithm~\ref{alg:greedy}) selects the exact brute-force-optimal coalition on
$74.0\%$ of tasks, and achieves $99.5\%$ of optimal utility on average, while
activating only $1.96$ agents per task versus $1.97$ for the true optimum and $8.00$
for full broadcast. This is consistent with the submodular-optimization connection
in Section~\ref{sec:marginal}: greedy addition performs far better in this instance
than the $1/2$ worst-case bound for unconstrained submodular maximization
\citep{buchbinder2015tight}, though we note this is an empirical observation on one
synthetic family of instances, not a general guarantee.

\paragraph{Shapley accuracy check.} For a representative task, we compute the
\emph{exact} Shapley value for all $8$ agents (feasible via the subset-sum formula at
this scale: $2^8=256$ coalitions) and compare it against a Monte Carlo estimate using
$3{,}000$ random-permutation samples per agent. The estimated and exact values agree
to within $0.005$ for every agent (mean absolute error $0.0028$, max $0.0046$), and a
repeated-sampling experiment ($30$ independent runs of $500$ permutations each) gives
a Monte Carlo standard error of about $0.007$ per agent, confirming the estimator is
accurate at this sample size. Separately, $\sum_i \hat\phi_i = 0.2239$ matches
$U(N) = 0.2239$ to four decimal places; this is the efficiency \emph{identity}
($\sum_i \phi_i = U(N)$ holds exactly for the true Shapley value by definition, and
holds up to floating-point precision for a Monte Carlo estimate that shares sampled
permutations across agents, because per-permutation marginal contributions telescope
to $U(N) - U(\emptyset)$) rather than independent evidence of estimator accuracy --
we report it only as a sanity check that the implementation is internally consistent,
not as a validation of Monte Carlo accuracy, which is established separately by the
exact-value comparison above. Agents whose primary skill matched a high-demand
dimension for that task received Shapley values roughly an order of magnitude larger
than agents whose skills were largely redundant with a stronger teammate -- the
qualitative behavior the credit-assignment mechanism of Section~\ref{sec:shapley} is
meant to produce.

\paragraph{Sandwich-bound check.} For the same task, we compute each agent's
diminishing-returns gap $\Gamma_i(x)$ (Equation~\ref{eq:gap}) and compare it against
the realized gap $|\Delta_i U(C,x) - \phi_i(U)|$ between the marginal contribution to
the full net utility (evaluated at the point Algorithm~\ref{alg:greedy} considers
agent $i$, using the exact Shapley value) and its fair credit. Corollary~\ref{cor:gap}
predicts $|\Delta_i U(C,x) - \phi_i(U)| \leq \Gamma_i(x)$ for every agent; this held
for all $8$ agents in the example ($0/8$ violations), with realized gaps ranging from
$0.014$ to $0.496$ against bounds ranging from $0.463$ to $0.653$ -- consistent with
the bound being correct but, as expected for a worst-case guarantee, not tight for
every agent.

\subsection{Sensitivity analysis}
\label{sec:sensitivity}

The result above is measured on one specific data-generating process --
specialized, low-overlap agents with a saturating quality function, linear
activation cost, and a quadratic redundancy penalty -- which was chosen because it
matches the running planner/coder/critic/security example, but which also
structurally favors small, non-redundant coalitions and is therefore favorable to our
method almost by construction. To assess how much of the $99.5\%$ result is a
property of Algorithm~\ref{alg:greedy} versus a property of this particular instance, we sweep
five parameters, holding $500 \to 200$ tasks per setting (reduced for the sweep to
keep total runtime reasonable) and reporting greedy utility as a percentage of the
brute-force optimum under each setting.

\begin{table}[t]
\centering
\caption{Sensitivity of greedy performance (as \% of brute-force-optimal utility) to
five parameters, each varied while holding the others at their Section~\ref{sec:simulation}
default. The submodularity-violation and estimation-noise sweeps are the most
informative: greedy degrades substantially once its core assumptions
(submodularity of $V$; accurate value estimates) are violated, while it is robust to
the other three axes.}
\label{tab:sensitivity}
\small
\begin{tabular}{@{}p{0.42\textwidth}p{0.5\textwidth}@{}}
\toprule
Sweep & Greedy utility as \% of brute-force optimum \\
\midrule
Skill overlap ($0\!\to\!1$, specialized $\to$ homogeneous)
  & $99.5,\ 99.9,\ 100.0,\ 100.0,\ 100.0$ (monotonically improves) \\[2pt]
Submodularity violation $\beta$ ($0\!\to\!0.8$, synergy bonus)
  & $99.6,\ 94.9,\ 80.2,\ 66.1,\ 68.4,\ 70.5$ (\emph{substantial degradation}) \\[2pt]
Value-estimation noise $\sigma$ ($0\!\to\!0.4$)
  & $99.5,\ 97.7,\ 94.9,\ 89.9,\ 86.2,\ 79.2$ (\emph{steady degradation}) \\[2pt]
Activation cost $c_i$ ($0\!\to\!0.4$)
  & $99.8,\ 99.5,\ 99.5,\ 99.9,\ 100.0,\ 100.0$ (robust) \\[2pt]
Redundancy weight $\lambda_4$ ($0\!\to\!0.8$)
  & $99.6,\ 99.7,\ 99.5,\ 99.7,\ 99.8,\ 99.7$ (robust) \\
\bottomrule
\end{tabular}
\end{table}

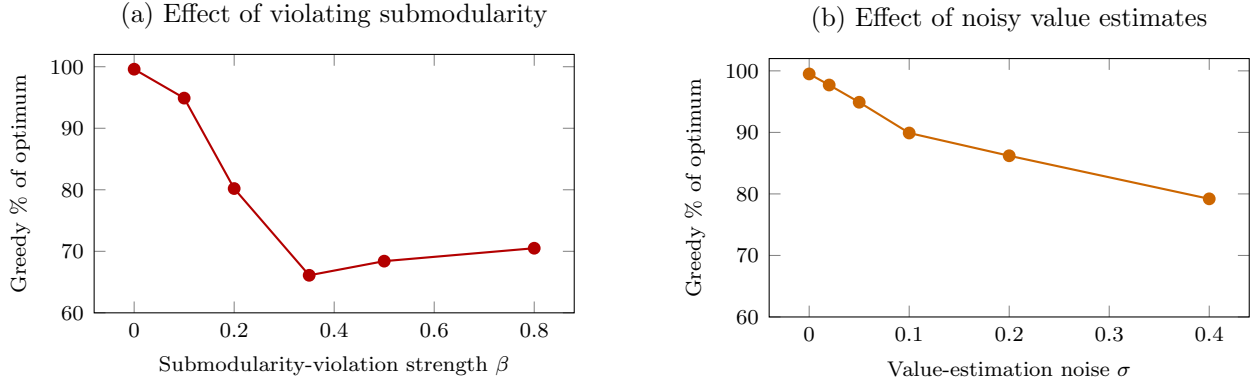
\begin{figure}[t]
\centering
\begin{tikzpicture}
\begin{axis}[
  width=0.48\textwidth, height=5cm,
  xlabel={Submodularity-violation strength $\beta$}, ylabel={Greedy \% of optimum},
  xlabel style={font=\scriptsize}, ylabel style={font=\scriptsize},
  tick label style={font=\scriptsize}, ymin=60, ymax=102,
  title={(a) Effect of violating submodularity}, title style={font=\small},
]
\addplot[mark=*, thick, color=red!70!black] coordinates {(0,99.6) (0.1,94.9) (0.2,80.2) (0.35,66.1) (0.5,68.4) (0.8,70.5)};
\end{axis}
\end{tikzpicture}
\hfill
\begin{tikzpicture}
\begin{axis}[
  width=0.48\textwidth, height=5cm,
  xlabel={Value-estimation noise $\sigma$}, ylabel={Greedy \% of optimum},
  xlabel style={font=\scriptsize}, ylabel style={font=\scriptsize},
  tick label style={font=\scriptsize}, ymin=60, ymax=102,
  title={(b) Effect of noisy value estimates}, title style={font=\small},
]
\addplot[mark=*, thick, color=orange!80!black] coordinates {(0,99.5) (0.02,97.7) (0.05,94.9) (0.1,89.9) (0.2,86.2) (0.4,79.2)};
\end{axis}
\end{tikzpicture}
\caption{The two axes along which greedy marginal-value routing degrades
substantially: (a) introducing pairwise synergy bonuses that violate the
diminishing-returns assumption underlying Algorithm~\ref{alg:greedy}, and (b)
adding noise to the value estimates the router uses to make decisions (while
scoring the resulting coalition on the true, noise-free value). Both are exactly
the failure modes Corollary~\ref{cor:gap} and Remark~\ref{rem:robustness}
anticipate: greedy is only as good as the submodularity of $V$ and the accuracy
of its value estimates.}
\label{fig:sensitivity}
\end{figure}

Three parameters -- skill overlap, activation cost, and redundancy weight -- leave
greedy's performance essentially unchanged (all $\geq 99.5\%$ of optimal across the
full sweep range), indicating the headline result is not narrowly tuned to one
choice of these values. The other two, shown in Figure~\ref{fig:sensitivity}, are
where greedy genuinely breaks down: (a) as we inject a pairwise synergy bonus that
rewards specific agent pairs for being \emph{jointly} present (violating the
diminishing-returns assumption that makes $V$ submodular), greedy's performance falls
from $99.6\%$ to as low as $66$--$70\%$ of optimal, because greedy's single-agent
marginal-value criterion cannot detect that two specific agents are worth more
together than the sum of their individual marginal contributions; (b) as we add noise
to the value estimate the router uses to make its decision (while scoring the
resulting coalition on the true, noise-free value), performance degrades smoothly from
$99.5\%$ at zero noise to $79.2\%$ at a noise standard deviation of $0.4$ (comparable
in scale to the marginal contributions themselves), showing the router's exact-match
rate with the optimum collapses much faster ($87.5\% \to 3.5\%$) than its utility
does, because near-ties in marginal value become easy to flip under noise even when
the resulting utility loss is small. We view these two sweeps -- not the single-point
$99.5\%$ headline number -- as the more informative characterization of the method:
Algorithm~\ref{alg:greedy} is a strong choice when $V$ is close to submodular and
reasonably well estimated, and a real deployment should monitor both properties (e.g.\
via the curvature diagnostic $\hat\alpha_Q$ of Section~\ref{sec:approx-guarantees} and
the per-agent gaps $\Gamma_i$ of Section~\ref{sec:sandwich}) rather than assume they
hold.

\subsection{A concrete empirical validation protocol}
\label{sec:validation}

The simulation above establishes that the optimization problem is well-posed and that
greedy routing is competitive with the optimum \emph{given} a characteristic
function of this form. It does not establish that Equation~\eqref{eq:v} is a good
model of real LLM agent quality, nor that the framework improves real pipelines. We
propose the following protocol for a follow-up empirical study:
\begin{enumerate}
  \item \textbf{Tasks and agent pool.} Use an existing multi-agent benchmark with a
  heterogeneous agent pool, e.g., software engineering tasks (SWE-bench-style, using
  planner/coder/critic/security agents) and multi-hop question answering (using
  planner/retriever/critic agents), so that skill complementarity is task-dependent.
  \item \textbf{Estimating $Q(C,x)$.} Approximate $Q$ with a learned or LLM-judge
  quality estimator, calibrated against ground-truth task success where available
  (e.g., unit-test pass rate for coding tasks), following the estimator-design
  choices used in comparable LLM-routing work \citep{ong2024routellm,yue2025masrouter}.
  \item \textbf{Baselines.} Compare the greedy marginal-value router against (a) a
  fixed dense-communication baseline, (b) a fixed hand-designed pipeline, (c) an
  existing learned multi-agent router such as MasRouter \citep{yue2025masrouter} or
  AgentPrune-style edge pruning \citep{zhang2025agentprune}.
  \item \textbf{Metrics.} Task success rate, total token cost, wall-clock latency, and
  -- specific to this framework -- the correlation between predicted Shapley
  estimates $\hat\phi_{i,t}(U)$ (Equation~\ref{eq:online-shapley}) and realized,
  post-hoc Shapley values, to test whether the online credit predictor is well
  calibrated.
\end{enumerate}
We view this protocol, rather than the synthetic simulation, as the appropriate
standard for claiming a real-world improvement, and we leave its execution to future
work.

\section{Limitations}
\label{sec:limitations}

Several limitations bound the current contribution. First, estimating $U(C\mid x)$
for every candidate coalition is combinatorially expensive in principle
($2^n$ coalitions); our greedy procedure and the submodularity connection
(Section~\ref{sec:marginal}) are proposed mitigations, but a full treatment requires
either a learned value-function approximator or a principled restriction to a small
lattice of plausible coalitions, neither of which we implement here. Second, every
guarantee in Sections~\ref{sec:approx-guarantees} and~\ref{sec:sandwich} --
the curvature bound, the double-greedy bound, and the Shapley--submodularity
sandwich -- assumes $U(\cdot\mid x)$ (or, for the curvature bound, $Q$) is exactly
submodular; real agent-quality functions may violate diminishing returns locally
(e.g., two verifier agents can be strictly complementary rather than redundant on
some tasks), and while Remark~\ref{rem:robustness} sketches how the guarantees
degrade gracefully under approximate submodularity via the submodularity ratio
\citep{das2011submodular}, we do not estimate this ratio empirically here (though the
sensitivity sweep in Section~\ref{sec:sensitivity} does show, empirically, how
performance degrades as submodularity is violated by an increasing synergy bonus).
Third, the Shapley value's classical fairness axioms additionally assume well-behaved
(e.g., superadditive) value functions for the core-stability results specifically;
once redundancy and conflict penalties are included, $U$ need not be superadditive,
and Section~\ref{sec:stability} shows that core-stability guarantees only transfer to
a restricted sub-domain -- characterizing stability for the general case remains
open. Fourth, the
mechanism-design discussion in Section~\ref{sec:mechanism} sketches a VCG-style target
rather than a fully specified, budget-balanced mechanism; making it concrete for LLM
agents with noisy, self-reported confidence is nontrivial. Fifth, the dynamic
extension (Section~\ref{sec:dynamic}) compounds credit assignment across a sequence of
coalition changes and is substantially harder than the static problem; we present it
as a direction, not a validated method. Finally, and most importantly, Section~\ref{sec:simulation} is a
synthetic, controlled simulation intended to validate the optimization problem's
structure, not a claim of improvement on real multi-agent LLM pipelines; we regard
the protocol in Section~\ref{sec:validation} as necessary before any deployment claim.

\section{Conclusion}
\label{sec:conclusion}

We formalized agent coalition formation and communication in agentic AI systems as a
cooperative game in which every message, skill activation, and coalition membership
must be justified by its expected marginal contribution net of cost. This
reformulation connects three previously separate concerns -- which agents to
activate, how they should communicate, and how to assign credit for what they
produced -- within a single characteristic function, and lets us bring stability and
incentive-alignment tools from cooperative game theory and mechanism design to bear
on agent orchestration. A controlled synthetic simulation shows the resulting
optimization problem is tractable in practice via a simple greedy heuristic, which
nearly matches a brute-force optimum while using a small fraction of the agents a
full-broadcast architecture would use. The central open question is empirical:
whether a characteristic function of the form in Equation~\eqref{eq:v}, estimated
from real task and agent data, yields routing and communication decisions that
improve real multi-agent LLM systems along the protocol proposed in
Section~\ref{sec:validation}. More agents, and more messages, do not automatically
create more net intelligence; the goal of this framework is to make that trade-off
explicit and optimizable.

\section*{Reproducibility}
The synthetic simulation in Section~\ref{sec:simulation} (agent generation, task
sampling, greedy router, brute-force search, exact and Monte Carlo Shapley
computation, and the sandwich-bound check) and the sensitivity analysis in
Section~\ref{sec:sensitivity} are released as supplementary code alongside this
submission (\texttt{sim.py} and \texttt{sim\_sensitivity.py}), with a fixed random
seed for exact reproducibility of every number reported in the paper. Code for the
empirical validation protocol in Section~\ref{sec:validation}, which requires access
to real LLM agent APIs, is left for the follow-up study that protocol describes.

\bibliographystyle{plainnat}
\bibliography{refs}

\end{document}